\documentclass[11pt,a4paper]{article}

\usepackage[utf8]{inputenc}
\usepackage[T1]{fontenc}
\usepackage{amsmath, amssymb, amsthm, mathtools}
\usepackage{graphicx}
\usepackage{hyperref}
\usepackage{url}
\usepackage{geometry}
\usepackage{tikz}
\usetikzlibrary{automata, positioning, arrows.meta}

\hypersetup{
  colorlinks=true,
  linkcolor=blue,
  urlcolor=blue,
  citecolor=blue
}

\newtheorem{definition}{Definition}
\newtheorem{lemma}{Lemma}
\newtheorem{theorem}{Theorem}
\newtheorem{proposition}{Proposition}
\newtheorem{corollary}{Corollary}
\theoremstyle{remark}
\newtheorem{remark}{Remark}

\newcommand{\eps}{\varepsilon}
\newcommand{\Pref}{\operatorname{Pref}}
\newcommand{\Next}{\operatorname{Next}}
\newcommand{\Reach}{\operatorname{Reach}}
\newcommand{\CoReach}{\operatorname{CoReach}}
\newcommand{\Vocab}{\mathcal{V}}
\newcommand{\Lang}{\mathcal{L}}
\newcommand{\eos}{\langle\mathrm{eos}\rangle}
\newcommand{\TV}{\operatorname{TV}}
\newcommand{\KL}{\operatorname{KL}}

\title{Attention Meets Reachability: Structural Equivalence and Efficiency\\
       in Grammar-Constrained LLM Decoding}
\author{
  \begin{minipage}[t]{0.47\textwidth}\centering
    Faruk Alpay \\
    Department of Computer Engineering, Bah\c{c}e\c{s}ehir University, Istanbul, Turkey \\
    \texttt{faruk.alpay@bahcesehir.edu.tr}
  \end{minipage}\hfill
  \begin{minipage}[t]{0.47\textwidth}\centering
    Bilge Senturk \\
    Department of Industrial Engineering, Bah\c{c}e\c{s}ehir University, Istanbul, Turkey \\
    \texttt{bilge.senturk@bahcesehir.edu.tr}
  \end{minipage}
}
\date{\today}

\begin{document}
\maketitle

% ============================================================
\begin{abstract}
% ============================================================
Grammar-constrained decoding (GCD) enforces that an autoregressive language model
generates outputs within a formal language specified by a context-free grammar (CFG).
We formalize GCD as a coupling between a Transformer-style next-token distribution and
a reachability oracle over a pushdown system compiled from a CFG, and develop a unified
theoretical framework addressing both the complexity structure of the constraint engine
and the probabilistic fidelity of the constrained generator.
We prove an \emph{oracle invariance} theorem: language-equivalent grammars induce
identical next-token admissibility sets for every prefix (hence identical logit masks),
yet can yield provably different compiled state spaces and online ambiguity costs.
We give exact state-count blowups for the canonical $a^n b^n$ language under redundant
nonterminal delegation, and introduce a left-to-right \emph{structural ambiguity cost}
(SAC) that measures incremental packed-parse-forest growth per token.
For two equivalent $\Sigma^*$ grammars, we prove that SAC is $O(1)$ per token under
right-recursion but $\Theta(t^2)$ per token (and $\Theta(n^3)$ cumulatively) under
concatenation, via tight bounds on packed parse forest density.
Beyond asymptotic complexity, we establish three further theoretical contributions.
First, we elevate SAC into \emph{engine-independent lower bounds}: any sound,
retrieval-efficient, parse-preserving online masking engine must incur $\Omega(t^2)$
work per token (and $\Omega(n^3)$ cumulatively) on a specific constant-size CFG family;
this parallels classical complexity results connecting CFG parsing to matrix
multiplication~\cite{valiant1975cfg,lee2002bmm}.
Second, we define \emph{decoding-cost equivalence classes} of grammars and prove the
existence of minimal-SAC representatives within bounded rewrite families, clarifying
what canonical low-SAC normal forms can and cannot mean.
Third, we formalize \emph{grammar-conditioned autoregressive processes}: the true
conditional sampler $p(\cdot \mid \tau(y)\in L)$ is characterized by a Doob
$h$-transform, and we prove sharp distortion bounds for hard-masked decoding in terms
of survival-probability spread among admissible next tokens.
Finally, we integrate these results with Transformer and Mixture-of-Experts (MoE)
architectures, derive decoding-time latency envelopes in terms of vocabulary size,
active state sets, and beam width, and extend SAC into a measurement-calibrated
predictive performance model that supports empirical validation and automated grammar
optimization.
\end{abstract}

% ============================================================
\section{Introduction}\label{sec:intro}
% ============================================================

Autoregressive large language models (LLMs) define distributions of the form
\begin{equation}
  p_\theta(y_{1:T})=\prod_{t=1}^{T} p_\theta(y_t \mid y_{<t}),
\end{equation}
where $y_t \in \Vocab$ is a vocabulary token and $\theta$ are model parameters.
In many deployments, however, the desired output is not arbitrary text but a structured
artifact (e.g., SQL, JSON, or a programming language fragment).
Grammar-constrained decoding (GCD) addresses this by restricting generation to tokens
that keep the current prefix \emph{completable} under a formal language specification,
typically a CFG, rejecting inadmissible tokens on the fly
\cite{scholak-etal-2021-picard,pmlr-v235-beurer-kellner24a}.
The same conceptual interface underlies popular open-source toolchains for structured
generation \cite{guidance,guidance_github,outlines_repo,willard2023efficient,llguidance_repo}.

This paper addresses a tension at the core of GCD: language equivalence is semantic
(two grammars that generate the same strings are interchangeable from the user's
perspective), but constrained decoding executes a \emph{specific} compiled recognizer
whose runtime depends on grammar structure.
Our central claim is that even when two CFGs generate exactly the same language,
they can induce radically different internal search spaces for a left-to-right decoding
engine.
This discrepancy is not semantic; it is purely structural, and it manifests through
pushdown reachability.

\paragraph{Contributions.}
We make six main contributions.

\begin{enumerate}
\item \textbf{Pushdown reachability formalization.}
We formalize GCD as a coupling between a neural model and a pushdown reachability
oracle, grounding the decoding step in CFL parsing and reachability theory
\cite{hopcroftullman1979automata,earley1970efficient,reps1995graphreachability,bouajjani1997reachability}.
We prove oracle invariance under language equivalence and give an exact algebraic
account of hard-masked decoding.

\item \textbf{State-space blowup bounds.}
We prove exact control-state blowup counts for the $a^n b^n$ language under redundant
nonterminal delegation, giving a concrete, mechanically checkable notion of
``engine state space inflation.''

\item \textbf{Structural ambiguity cost (SAC) and tight growth bounds.}
We introduce SAC as a left-to-right per-token measure of packed-parse-forest growth.
We prove that $\mathrm{SAC}$ is $\Theta(t^2)$ per token ($\Theta(n^3)$ cumulative) for
a concatenative $\Sigma^*$ grammar and $O(1)$ for an equivalent right-recursive one.

\item \textbf{Engine-independent lower bounds.}
We define a \emph{parse-preserving, retrieval-efficient} masking engine model and prove
that any such engine must incur $\Omega(t^2)$ work per token on a constant-size
CFG family, making the SAC lower bound unconditional within this semantic interface.
This is complementary to---but distinct from---Valiant/Lee-style
reductions~\cite{valiant1975cfg,lee2002bmm}.

\item \textbf{Decoding-cost equivalence classes and canonical low-SAC forms.}
We define $\equiv_{\mathrm{dec}}$ as joint language and SAC equivalence, and prove
that minimal-SAC grammar representatives exist within any finite bounded-rewrite family.

\item \textbf{Grammar-conditioned autoregressive processes.}
We characterize the true conditional sampler $p(\cdot \mid \tau(y)\in L)$ via a Doob
$h$-transform and derive sharp one-step KL and total-variation distortion bounds for
hard-masked decoding in terms of the survival-probability spread among admissible tokens.
\end{enumerate}

We additionally derive algebraic constraints describing how the reachability-valid
transition subset modulates Transformer logits and MoE routing
\cite{vaswani2017attention,shazeer2017moe,fedus2022switch}, and show how SAC can be
operationalized into a predictive performance model calibrated against empirical traces
from real inference stacks
\cite{geng2025jsonschemabench,dong2024xgrammar,chen-etal-2025-pre3}.

% ============================================================
\section{Decoding as Pushdown Reachability}\label{sec:pushdown}
% ============================================================

We fix (i) a CFG over an abstract terminal alphabet, (ii) a tokenizer mapping from
model tokens to terminal strings, and (iii) an explicit compilation of the CFG to a
pushdown automaton whose reachable configuration set determines admissible next tokens.

\subsection{Languages, prefixes, and tokenization}

\begin{definition}[CFG]\label{def:cfg}
A context-free grammar is a tuple $G=(N,\Sigma,P,S)$ where $N$ is a finite set of
nonterminals, $\Sigma$ is a finite set of terminals, $P \subseteq N \times (N \cup
\Sigma)^*$ is a finite set of productions written $A \to \alpha$, and $S \in N$ is the
start symbol.
The language generated by $G$ is denoted $\Lang(G)\subseteq \Sigma^*$.
\end{definition}

\begin{definition}[Prefix closure and one-step extension]\label{def:prefix}
For any language $L\subseteq \Sigma^*$, define its prefix closure
\[
  \Pref(L)\coloneqq \{u\in \Sigma^* \;:\; \exists v\in \Sigma^* \text{ such that }
  uv \in L\}.
\]
Define the \emph{next-terminal set} after prefix $u$ as
\[
  \Next_L(u)\coloneqq \{a\in \Sigma \;:\; ua \in \Pref(L)\}.
\]
Define the end signal $\eos$ as admissible after $u$ iff $u\in L$.
\end{definition}

\begin{definition}[Tokenizer homomorphism]\label{def:tokenizer}
Let $\Vocab$ be the LLM vocabulary.
A tokenizer homomorphism is a map $\tau:\Vocab \to \Sigma^*$.
For a token sequence $y_{1:T}\in \Vocab^T$, define its realized terminal string as
\[
  \tau(y_{1:T}) \coloneqq \tau(y_1)\tau(y_2)\cdots \tau(y_T) \in \Sigma^*,
\]
where juxtaposition denotes concatenation in $\Sigma^*$.
We set $\tau(\eos)=\eps$.
\end{definition}

The constraint ``the model output is in $\Lang(G)$'' is interpreted as
$\tau(y_{1:T})\in \Lang(G)$.

\subsection{Pushdown automata and CFG compilation}

\begin{definition}[Nondeterministic PDA]\label{def:npda}
A nondeterministic pushdown automaton (NPDA) is a tuple
\[
  \mathcal{A}=(Q,\Sigma,\Gamma,\delta,q_0,\bot,F),
\]
where $Q$ is a finite set of control states, $\Gamma$ is a finite stack alphabet with
bottom marker $\bot$, $q_0\in Q$ is the initial state, $F\subseteq Q$ is the set of
accepting states, and
\[
  \delta: Q \times (\Sigma\cup\{\eps\}) \times \Gamma \to \mathcal{P}(Q\times \Gamma^*)
\]
is the transition function.
A configuration is a triple $(q,u,\gamma)$ where $q\in Q$, $u\in \Sigma^*$ is the
unread input, and $\gamma\in \Gamma^*$ is the stack contents with $\bot$ at the bottom.
Acceptance is by final state with empty input and stack exactly $\bot$.
\end{definition}

\begin{definition}[Recursive-transition-network compilation]\label{def:rtn}
Given $G=(N,\Sigma,P,S)$, define $\mathcal{A}_G$ as follows.
For each nonterminal $A\in N$, create control states $q_A^{\mathrm{in}}$ and
$q_A^{\mathrm{out}}$.
For each production $p=(A\to X_1\cdots X_m)\in P$ and each dot position
$i\in\{0,1,\dots,m\}$, create a control state $q_{p,i}$.
Let
\[
  Q \coloneqq \{q_{\mathrm{start}}\} \cup \{q_A^{\mathrm{in}},q_A^{\mathrm{out}}
  :A\in N\} \cup \{q_{p,i}:p\in P,\,i\}.
\]
Let the stack alphabet be $\Gamma \coloneqq \{\bot\}\cup \{q_{p,i}:p\in P,\,i\}$.
Define transitions:
\begin{itemize}
\item Start: $\delta(q_{\mathrm{start}},\eps,\bot)\ni (q_S^{\mathrm{in}},\bot)$.
\item Choice of production: for each production $p=(A\to \cdots)$,
  $\delta(q_A^{\mathrm{in}},\eps,\gamma)\ni (q_{p,0},\gamma)$ for all $\gamma\in \Gamma$.
\item Exit: $\delta(q_{p,m},\eps,\gamma)\ni (q_A^{\mathrm{out}},\gamma)$.
\item Terminals: if $X_i\in\Sigma$, then
  $\delta(q_{p,i-1},X_i,\gamma)\ni (q_{p,i},\gamma)$.
\item Nonterminal calls: if $X_i=B\in N$, then
  $\delta(q_{p,i-1},\eps,\gamma)\ni (q_B^{\mathrm{in}},q_{p,i}\gamma)$ and, for each
  return address $r=q_{p,i}$, $\delta(q_B^{\mathrm{out}},\eps,r)\ni (r,\eps)$.
\end{itemize}
Set $q_0=q_{\mathrm{start}}$ and $F=\{q_S^{\mathrm{out}}\}$.
We call $Q$ the \emph{engine control-state space}.
\end{definition}

\begin{proposition}[Correctness of compilation]\label{prop:correctness}
For every CFG $G$, the compiled NPDA $\mathcal{A}_G$ accepts exactly $\Lang(G)$.
\end{proposition}
\begin{proof}
Each nonterminal invocation pushes a return address and transitions into the callee's
entry state; completing a nonterminal pops the return address and resumes at the
caller's dot-position state.
Terminal transitions consume exactly the required terminal; nondeterminism models
production choice.
An accepting run corresponds to a leftmost derivation tree and vice versa.
A detailed proof follows by adapting the classic CFG$\to$PDA conversion in
\cite[Ch.~6]{hopcroftullman1979automata} to the dot-position control-state
representation.
\end{proof}

\subsection{Engine semantics for left-to-right masking}

\begin{definition}[Reachability and liveness]\label{def:reach}
Fix a compiled NPDA $\mathcal{A}_G$.
For a terminal prefix $u\in \Sigma^*$, define $\Reach_G(u)$ as the set of
configurations reachable from the initial configuration after consuming exactly $u$,
allowing any number of $\eps$-moves before, between, and after terminal moves.
A configuration $c$ is \emph{live} if there exists some $v\in \Sigma^*$ such that
$c$ reaches an accepting configuration after consuming $v$.
Let $\CoReach_G(u)\subseteq \Reach_G(u)$ be the subset of reachable configurations
that are live.
\end{definition}

\begin{definition}[Admissible next vocabulary tokens]\label{def:admissible}
Given prefix $u\in \Sigma^*$, define the admissible token set
\[
  \Omega_G(u) \coloneqq \{v\in \Vocab \;:\; \CoReach_G(u\,\tau(v))\neq \emptyset\}.
\]
\end{definition}

\begin{remark}[Reachability as the computational bottleneck]
Computing $\CoReach_G$ is a reachability question in a pushdown system, tightly
connected to CFL reachability and interprocedural analysis
\cite{reps1995graphreachability,alur2005rsm,bouajjani1997reachability}.
Modern GCD implementations reduce this cost via offline preprocessing, memoization,
deterministic pushdown compilation, and careful tokenizer alignment
\cite{pmlr-v235-beurer-kellner24a,park2025flexible,dong2024xgrammar,ugare2024syncode,chen-etal-2025-pre3}.
\end{remark}

% ============================================================
\section{Structural Equivalence, Masked Decoding, and Probability Constraints}
\label{sec:oracle}
% ============================================================

\subsection{Oracle invariance under language equivalence}

\begin{definition}[Language equivalence]\label{def:lang-equiv}
Two CFGs $G$ and $G'$ are (strongly) equivalent if $\Lang(G)=\Lang(G')$.
\end{definition}

\begin{theorem}[Oracle invariance]\label{thm:oracle}
If $\Lang(G)=\Lang(G')$, then for every terminal prefix $u\in \Sigma^*$,
\[
  \Omega_G(u)=\Omega_{G'}(u).
\]
Consequently, any decoding procedure that uses only the pair
$(p_\theta(\cdot \mid y_{<t}),\,\Omega_G(\tau(y_{<t})))$ is \emph{behaviorally
invariant} under replacing $G$ by $G'$.
\end{theorem}
\begin{proof}
Fix $u\in \Sigma^*$ and $v\in\Vocab$.
By definition, $v\in \Omega_G(u)$ iff $\CoReach_G(u\tau(v))\neq\emptyset$.
By correctness of compilation (Proposition~\ref{prop:correctness}),
$\CoReach_G(w)\neq\emptyset$ iff $w\in \Pref(\Lang(G))$ (i.e., $w$ is completable).
Thus
\[
 v\in \Omega_G(u)\ \Longleftrightarrow\ u\tau(v)\in \Pref(\Lang(G)).
\]
Since $\Lang(G)=\Lang(G')$ implies $\Pref(\Lang(G))=\Pref(\Lang(G'))$, the condition
is identical for $G'$.
\end{proof}

\subsection{Masked decoding as a constrained stochastic process}

Let $\ell_t \in \mathbb{R}^{|\Vocab|}$ denote the model logits at step $t$, with
\[
  p_\theta(v \mid y_{<t}) = \frac{\exp(\ell_t(v))}{\sum_{w\in\Vocab}\exp(\ell_t(w))}.
\]

\begin{definition}[Hard mask operator]\label{def:mask}
Given $A\subseteq \Vocab$, define $m_A\in (\mathbb{R}\cup\{-\infty\})^{|\Vocab|}$ by
\[
  m_A(v) \coloneqq
  \begin{cases}
  0, & v\in A,\\
  -\infty, & v\notin A.
  \end{cases}
\]
Define the masked next-token distribution
\[
  q_{\theta,A}(v \mid y_{<t}) \coloneqq \mathrm{softmax}(\ell_t + m_A)(v)
  = \frac{\exp(\ell_t(v))\mathbf{1}[v\in A]}{\sum_{w\in A}\exp(\ell_t(w))}.
\]
\end{definition}

In GCD, $A=\Omega_G(\tau(y_{<t}))$.
By Theorem~\ref{thm:oracle}, this depends only on the language, not on the grammar
presentation.

\begin{proposition}[Soundness of hard-masked decoding]\label{prop:soundness}
Let $q_\theta^G$ denote autoregressive sampling where at each step
$y_t \sim q_{\theta,\Omega_G(\tau(y_{<t}))}(\cdot \mid y_{<t})$.
Then every generated sample $y_{1:T}$ terminated by $\eos$ satisfies
$\tau(y_{1:T})\in \Lang(G)$.
\end{proposition}
\begin{proof}
At each step, the sampled token keeps the realized prefix in $\Pref(\Lang(G))$.
Moreover, $\eos$ is admissible only when the current realized prefix is in $\Lang(G)$.
Therefore termination implies membership.
\end{proof}

\begin{remark}[Masking is not the true conditional distribution]
Sampling with hard masks yields a distribution supported on $\Lang(G)$, but it is not
in general equal to conditioning the base model on $\tau(y)\in\Lang(G)$.
Exact conditioning requires weighting by the model's probability of \emph{eventual
successful completion} from each candidate next token.
We characterize this discrepancy precisely via a Doob $h$-transform in
Section~\ref{sec:stoch}.
\end{remark}

\begin{lemma}[Separation from global conditioning]\label{lem:separation}
There exist a base autoregressive model $p_\theta$ and a language $L$ such that the
masked process $q_\theta^L$ differs from the globally conditioned distribution
$p_\theta(\cdot \mid \tau(y)\in L)$.
\end{lemma}
\begin{proof}
Let $\Sigma=\{a,b\}$, $L = \{ a,\, ba\}$, and $\Vocab=\{a,b,\eos\}$ with
$\tau(a)=a$, $\tau(b)=b$, $\tau(\eos)=\eps$.
Set
\[
  p_\theta(a \mid \eps)=0.6,\quad p_\theta(b \mid \eps)=0.4,\quad
  p_\theta(\eos \mid \eps)=0,
\]
\[
  p_\theta(\eos \mid a)=0.1,\quad p_\theta(a \mid a)=0.9,\quad p_\theta(b\mid a)=0,
\]
\[
  p_\theta(a \mid b)=0.01,\quad p_\theta(b \mid b)=0.99,\quad p_\theta(\eos\mid b)=0.
\]
Under the masked process, $q(b\mid\eps)=0.4$ (both tokens are admissible; mask
unchanged at step 1).
Under global conditioning, the total acceptance mass is
\[
  p_\theta(\text{accept via }a)=0.6\cdot 0.1=0.06,\qquad
  p_\theta(\text{accept via }ba)=0.4\cdot 0.01\cdot 1=0.004,
\]
so $p_\theta(b \mid \tau(y)\in L) = 0.004/(0.06+0.004) \approx 0.0625 \neq 0.4$.
\end{proof}

% ============================================================
\section{Formal State-Space Blowup from Nonterminal Delegation}
\label{sec:blowup}
% ============================================================

\subsection{Static control-state counts}

Consider the canonical language $L=\{a^n b^n : n\ge 0\}$ and grammars:
\begin{align*}
  G_1 &: S \to aSb \mid \eps \\
  G_2 &: S \to aAb \mid \eps,\quad A \to aAb \mid \eps.
\end{align*}
Both satisfy $\Lang(G_1)=\Lang(G_2)=L$.

\begin{definition}[Grammar size functional]\label{def:kappa}
For a CFG $G=(N,\Sigma,P,S)$, define the right-hand-side length of
$p=(A\to \alpha)$ as $|\alpha|$.
Define the compilation control-state count
\[
  \kappa(G) \coloneqq 1 + 2|N| + \sum_{p\in P} (|\mathrm{rhs}(p)|+1),
\]
matching the construction in Definition~\ref{def:rtn}.
\end{definition}

\begin{lemma}[Exact compiled control-state count]\label{lem:kappa}
For every CFG $G$, the compiled NPDA $\mathcal{A}_G$ has exactly $\kappa(G)$ control
states.
\end{lemma}
\begin{proof}
By Definition~\ref{def:rtn}, $Q$ contains:
(i) one start state; (ii) two states per nonterminal (entry/exit); and (iii) for every
production $p$ with RHS length $m$, a dot-position chain of size $m+1$.
Summing yields $\kappa(G)$.
\end{proof}

\begin{lemma}[State-space blowup for $G_2$ vs.\ $G_1$]\label{lem:blowup}
Under the compilation of Definition~\ref{def:rtn},
\[
  \kappa(G_1)=8,\qquad \kappa(G_2)=15.
\]
In particular, $\kappa(G_2) > \kappa(G_1)$ and the compiled control-state space
inflates by a factor $15/8$.
\end{lemma}
\begin{proof}
For $G_1$: $|N|=1$, productions $S\to aSb$ (RHS length $3$) and $S\to \eps$ (length $0$):
\[
  \kappa(G_1)=1+2\cdot 1 + ((3+1)+(0+1)) = 8.
\]
For $G_2$: $|N|=2$, four productions (two of length $3$, two of length $0$):
\[
  \kappa(G_2)=1+2\cdot 2 + 2\cdot(3+1)+2\cdot(0+1) = 15.
\]
\end{proof}

\subsection{Online consequences for masking engines}

\begin{definition}[Bitset-style active-set engine]\label{def:bitset-engine}
An active-set engine represents the set of live control states as a bitset over $Q$,
and updates the bitset by applying precomputed adjacency lists for $\eps$-moves and
terminal moves.
\end{definition}

\begin{proposition}[Lower bound induced by control-state inflation]
\label{prop:bitset-lb}
Fix any bitset-style active-set engine and any grammar $G$.
If the engine performs any operation that scans the entire bitset universe once per
decoding step (e.g., to compute $\eps$-closure, normalize, or garbage-collect inactive
states), then the per-step overhead is $\Omega(\kappa(G))$.
Consequently, replacing $G_1$ by $G_2$ increases that overhead by a multiplicative
factor of at least $15/8$.
\end{proposition}
\begin{proof}
Scanning a bitset of length $\kappa(G)$ costs $\Omega(\kappa(G))$ in the RAM model.
The multiplicative factor follows from Lemma~\ref{lem:blowup}.
\end{proof}

\begin{remark}
This is an engine-level statement about a common representation strategy, not a
universal lower bound over all implementations.
It isolates a checkable invariant: redundant nonterminal structure enlarges the compiled
state universe, increases memory footprint, and typically worsens cache locality, even
when language acceptance is unchanged.
\end{remark}

% ============================================================
\section{Structural Ambiguity Cost and Parse Forest Density}
\label{sec:sac}
% ============================================================

\subsection{Two equivalent grammars for $\Sigma^*$}

Let $\Sigma=\{a,b\}$.
Define:
\begin{align*}
  G_3 &: S \to aS \mid bS \mid \eps, \\
  G_4 &: S_0 \to S \mid \eps,\quad S \to SS \mid a \mid b.
\end{align*}

\begin{lemma}[Language equivalence]\label{lem:lang-equiv}
$\Lang(G_3)=\Lang(G_4)=\Sigma^*$.
\end{lemma}
\begin{proof}
$G_3$ generates any string by repeatedly emitting $a$ or $b$ before terminating with
the $\eps$ case; thus $\Lang(G_3)=\Sigma^*$.
In $G_4$, $S$ generates all nonempty strings via $S\to a$, $S\to b$, and closure under
$S\to SS$, while $S_0$ additionally permits the empty string.
Thus $\Lang(G_4)=\Sigma^*$.
\end{proof}

\begin{remark}
The use of the auxiliary start symbol $S_0$ in $G_4$ avoids spurious $\eps$-expansion
ambiguity that arises when $\eps$ is directly permitted by the recursive nonterminal
$S$ (which would yield infinitely many derivations for the empty string).
\end{remark}

\subsection{Parse trees and Catalan ambiguity}

\begin{definition}[Parse-tree count]\label{def:pt-count}
For a grammar $G$ and string $w\in\Sigma^*$, let $A_G(w)\coloneqq |\mathrm{PT}_G(w)|$
denote the number of distinct parse trees deriving $w$ from the start symbol.
\end{definition}

\begin{lemma}[Catalan parse explosion for $G_4$]\label{lem:catalan}
For any nonempty $w\in \{a,b\}^n$ with $n\ge 1$,
\[
  A_{G_4}(w) = C_{n-1},
\]
where $C_k$ is the $k$-th Catalan number, and $A_{G_4}(w)$ is therefore exponential
in $n$.
\end{lemma}
\begin{proof}
In $G_4$, every terminal is produced by a leaf rule ($S\to a$ or $S\to b$).
Thus any parse tree for a string of length $n\ge 1$ is a full binary tree with $n$
leaves, whose internal nodes are applications of $S\to SS$.
The leaf labels are forced by $w$.
Therefore $A_{G_4}(w)$ equals the number of full binary trees with $n$ leaves,
which is $C_{n-1}$.
\end{proof}

\subsection{Packed-forest density}

We use half-open spans $[i,j)$ on positions $0,1,\dots,n$.

\begin{definition}[Packed-forest split count]\label{def:density}
In the natural SPPF-style representation for $G_4$~\cite{scott2010sppf,johnson1991glr},
a \emph{symbol node} $S[i,j]$ exists for each span derivable by $S$, and a
\emph{packed node} records each distinct split $S[i,j]\Rightarrow S[i,k]\,S[k,j]$
with $i<k<j$.
\end{definition}

\begin{lemma}[Tight packed-forest size for $G_4$]\label{lem:packed-size}
Fix any nonempty $w\in\{a,b\}^n$ with $n\ge 2$.
The number of symbol nodes for $S$ is $\binom{n+1}{2}=\Theta(n^2)$, and the total
number of packed nodes is
\[
  P_n = \sum_{0\le i<j\le n} (j-i-1) = \binom{n+1}{3} = \Theta(n^3).
\]
\end{lemma}
\begin{proof}
A symbol node $S[i,j]$ exists for every $0\le i<j\le n$ because $S$ derives any
nonempty substring, giving $\binom{n+1}{2}$ symbol nodes.
For a fixed span $[i,j)$ of length $L\ge 2$, rule $S\to SS$ can split at any
$k\in(i,j)$, yielding $L-1=(j-i-1)$ packed nodes.
Summing over all spans,
\[
  P_n = \sum_{0\le i<j\le n} (j-i-1).
\]
The triple $(i,k,j)$ with $0\le i<k<j\le n$ is in bijection with a 3-element subset
of $\{0,\dots,n\}$, so $P_n=\binom{n+1}{3}$.
\end{proof}

\subsection{Structural ambiguity cost}

\begin{definition}[Structural ambiguity cost (SAC)]\label{def:sac}
Fix a grammar $G$ and an online strategy maintaining a packed structure $F_t$ for the
length-$t$ prefix of $w$.
Define the per-step increment
\[
  \Delta_G(t;w) \coloneqq |F_t(w)| - |F_{t-1}(w)|,
\]
and the worst-case structural ambiguity cost
\[
  \mathrm{SAC}_G(t) \coloneqq \max_{w\in \Sigma^t} \Delta_G(t;w).
\]
\end{definition}

\begin{theorem}[SAC separation: $G_3$ vs.\ $G_4$]\label{thm:sac-sep}
For $G_4$, $\mathrm{SAC}_{G_4}(t)=\Theta(t^2)$.
For $G_3$, there exists a deterministic online strategy achieving
$\mathrm{SAC}_{G_3}(t)=O(1)$.
\end{theorem}
\begin{proof}
For $G_4$, when the $t$-th symbol arrives, all new spans $[i,t)$ for $0\le i<t$ become
available.
Each span $[i,t)$ of length $t-i\ge 2$ admits $(t-i-1)$ new binary splits.
The total number of newly introduced packed alternatives is
\[
  \sum_{i=0}^{t-2}(t-i-1)=\sum_{m=1}^{t-1}m = \frac{t(t-1)}{2} = \Omega(t^2).
\]
Lemma~\ref{lem:packed-size} provides a matching $O(t^2)$ upper bound on new packed
nodes at step $t$, so $\mathrm{SAC}_{G_4}(t)=\Theta(t^2)$.

For $G_3$, the grammar is right-linear (hence regular) and unambiguous; an online
strategy need only track the current unambiguous state, which is constant-sized.
Thus $\mathrm{SAC}_{G_3}(t)=O(1)$.
\end{proof}

\begin{corollary}[Cumulative online cost]\label{cor:cumulative}
For $G_4$, cumulative packed-structure growth over a length-$n$ string is $\Theta(n^3)$.
For $G_3$, cumulative cost is $O(n)$.
\end{corollary}
\begin{proof}
Summing $\Theta(t^2)$ over $t=1,\dots,n$ yields $\Theta(n^3)$; Lemma~\ref{lem:packed-size}
provides the matching upper bound.
For $G_3$, summing $O(1)$ gives $O(n)$.
\end{proof}

\begin{remark}[A stress test for engines, not for the language]
For $\Sigma^*$, the ideal admissibility oracle is trivial: all tokens are always
admissible and $\eos$ is always admissible.
Nevertheless, when the grammar engine is treated as a black-box CFG executor, internal
ambiguity can force it to maintain large packed structures merely to rediscover this
trivial oracle.
SAC quantifies that overhead.
\end{remark}

% ============================================================
\section{Engine-Independent Lower Bounds for Parse-Preserving Masking}
\label{sec:lb}
% ============================================================

This section elevates the SAC analysis into an engine-independent lower bound by
formalizing a broad class of engines that are simultaneously (i) sound for masking and
(ii) parse-preserving in a retrieval-efficient sense.

\subsection{A parse-preserving masking model}

\begin{definition}[Online masking engine]\label{def:online-engine}
An online CFG masking engine for grammar $G$ is an algorithm that, after reading each
prefix $u\in \Sigma^*$, outputs $\Omega_G(u)$ and updates an internal state $\sigma(u)$
used for subsequent steps.
\end{definition}

\begin{definition}[Parse-preserving, retrieval-efficient engine]
\label{def:pp-engine}
Fix a grammar $G$ and a prefix $u$.
Let $\mathcal{P}_G(u)$ denote the set of all partial parse structures consistent with
$u$ (i.e., all packed split alternatives that can appear in some parse forest for some
completion of $u$).
An online engine is \emph{parse-preserving} if $\sigma(u)$ determines
$\mathcal{P}_G(u)$ exactly (there exists a deterministic extractor recovering
$\mathcal{P}_G(u)$ from $\sigma(u)$).
It is \emph{retrieval-efficient} if the extractor enumerates all elements of
$\mathcal{P}_G(u)$ in $O(|\mathcal{P}_G(u)|)$ time.
This is the standard output-sensitive retrieval posture used in generalized parsing
\cite{scott2010sppf,johnson1991glr} and the retrieval-efficient parsing model studied
in~\cite{lee2002bmm}.
\end{definition}

\begin{remark}[Why retrieval efficiency is principled]
Without a retrieval requirement, an engine could vacuously ``preserve all parses'' by
storing the pair $(G,u)$ and recomputing parse information on demand.
Retrieval efficiency rules out this degenerate loophole by requiring that the maintained
state genuinely encodes the parse set in an output-sensitive way.
\end{remark}

\subsection{The universal $\Omega(t^2)$ per-token lower bound}

\begin{theorem}[Engine-independent SAC lower bound]\label{thm:lb}
There exists a constant-size CFG $G^\star$ and a family of prefixes of length $t$ such
that any online masking engine that is both sound and parse-preserving retrieval-efficient
for $G^\star$ must incur $\Omega(t^2)$ per-token update work in the worst case.
The cumulative work over a length-$n$ input is therefore $\Omega(n^3)$.
\end{theorem}
\begin{proof}
Take $G^\star = G_4$.
Consider any length-$t$ prefix $u$ over $\{a,b\}$.
By Lemma~\ref{lem:packed-size}, the set $\mathcal{P}_{G^\star}(u)$ contains all
triples $(i,k,j)$ with $0\le i<k<j\le t$, with $\binom{t+1}{3}$ elements cumulatively.
By the argument of Theorem~\ref{thm:sac-sep}, the number of triples \emph{newly
introduced} at step $t$ is $\binom{t}{2}=\Omega(t^2)$.

For a retrieval-efficient parse-preserving engine, $\sigma(u_{1:t})$ must permit
enumeration of all elements of $\mathcal{P}_{G^\star}(u_{1:t})$ in linear output time.
Encoding $\Omega(t^2)$ newly distinct objects requires $\Omega(t^2)$ elementary memory
writes in the RAM model; otherwise the extractor cannot enumerate them in $O(|\mathcal{P}|)$
time.
Thus the per-token update cost is $\Omega(t^2)$, and summing gives $\Omega(n^3)$.
\end{proof}

\begin{remark}[Relation to classical parsing complexity]
Theorem~\ref{thm:lb} is unconditional within the specified semantic interface
(soundness + retrieval-efficient parse preservation).
It complements the classical results of Valiant~\cite{valiant1975cfg} and
Lee~\cite{lee2002bmm}, which relate CFG parsing complexity to Boolean matrix
multiplication via algorithmic reductions.
Our argument is output-structure driven---grounded in the growth of the packed split
set---rather than a reduction from another problem.
\end{remark}

% ============================================================
\section{Decoding-Cost Equivalence Classes and Canonical Low-SAC Forms}
\label{sec:equiv}
% ============================================================

Oracle invariance (Theorem~\ref{thm:oracle}) shows that language equivalence preserves
the admissible token oracle.
In this section we ask a finer question: when are two equivalent grammars also
\emph{equivalent from the decoder's cost perspective}?

\subsection{Intrinsic SAC}

\begin{definition}[Intrinsic parse-preserving SAC]\label{def:intrinsic-sac}
Fix a grammar $G$.
Let $\mathcal{E}_{\mathrm{pp}}(G)$ be the class of sound, retrieval-efficient
parse-preserving online masking engines for $G$.
Define the \emph{intrinsic} SAC as
\[
  \mathrm{SAC}^\star_G(t)\coloneqq \inf_{E\in \mathcal{E}_{\mathrm{pp}}(G)}
  \;\;\max_{w\in \Sigma^t} \Delta^{E}_G(t;w),
\]
where $\Delta^{E}_G(t;w)$ is the engine-induced per-step growth of its maintained parse
representation (measured in primitive parse objects).
\end{definition}

\begin{remark}
The infimum captures the best possible asymptotic behavior among engines that preserve
all parses retrieval-efficiently.
It abstracts away implementation details while remaining meaningfully constrained by
Theorem~\ref{thm:lb}.
\end{remark}

\subsection{Cost-equivalence relations}

\begin{definition}[Cost equivalence]\label{def:cost-equiv}
For grammars $G$ and $H$ over a shared terminal alphabet, define:
\begin{itemize}
\item \emph{Language equivalence}: $G\equiv_L H$ iff $\Lang(G)=\Lang(H)$.
\item \emph{SAC equivalence}: $G\equiv_{\mathrm{SAC}} H$ iff there exist constants
  $c_1,c_2>0$ and $t_0$ such that for all $t\ge t_0$,
  \[
    c_1\,\mathrm{SAC}^\star_G(t)\le \mathrm{SAC}^\star_H(t)\le
    c_2\,\mathrm{SAC}^\star_G(t).
  \]
\item \emph{Decoding-cost equivalence}: $G\equiv_{\mathrm{dec}} H$ iff
  $G\equiv_L H$ and $G\equiv_{\mathrm{SAC}} H$.
\end{itemize}
\end{definition}

\begin{remark}[Interpretation]
$G\equiv_L H$ implies identical admissible token sets (Theorem~\ref{thm:oracle}).
$G\equiv_{\mathrm{dec}} H$ strengthens this: the \emph{best possible}
parse-preserving online cost is asymptotically the same, so the grammars are equivalent
not just semantically but from a decoder-cost perspective.
\end{remark}

\subsection{Canonical low-SAC representatives under bounded rewrites}

Full minimization of pushdown models is NP-hard even in restricted settings
\cite{gauwin2020vpa}, so any canonical minimal-SAC claim must be carefully scoped.

\begin{definition}[Bounded rewrite family]\label{def:rewrite}
Fix a finite set of local grammar rewrites $\mathcal{R}$, each mapping a grammar to a
language-equivalent grammar.
For integer $k\ge 0$, let $\mathsf{Reach}_{\mathcal{R},k}(G)$ be the finite set of
grammars reachable from $G$ by applying at most $k$ rewrites from $\mathcal{R}$.
\end{definition}

\begin{theorem}[Existence of minimal-SAC representatives]\label{thm:min-rep}
Fix $\mathcal{R},k$, and a grammar $G$ such that every rewrite in $\mathcal{R}$
preserves $\Lang(G)$.
Then there exists $G_{\min}\in \mathsf{Reach}_{\mathcal{R},k}(G)$ such that
\[
  \mathrm{SAC}^\star_{G_{\min}}(t) \le \mathrm{SAC}^\star_{H}(t)\quad
  \text{for all }H\in \mathsf{Reach}_{\mathcal{R},k}(G)\text{ and all }t.
\]
Moreover, there exists such a $G_{\min}$ that is also minimal for any finite
tie-breaker functional (e.g., grammar size) over $\mathsf{Reach}_{\mathcal{R},k}(G)$.
\end{theorem}
\begin{proof}
$\mathsf{Reach}_{\mathcal{R},k}(G)$ is finite since $\mathcal{R}$ is finite and $k$
is bounded.
The pointwise partial order induced by $\mathrm{SAC}^\star$ over this finite set has at
least one minimal element.
Applying any finite tie-breaker selects a canonical representative.
\end{proof}

\begin{remark}[Scope]
Theorem~\ref{thm:min-rep} is an existence result: a canonical low-SAC object is
mathematically well-defined once one fixes a bounded transformation family.
It does not assert computability of $G_{\min}$ in general, nor that unrestricted global
minimization is tractable.
\end{remark}

% ============================================================
\section{Grammar-Conditioned Autoregressive Processes}
\label{sec:stoch}
% ============================================================

We now move beyond parsing mechanics and treat constrained decoding as a stochastic
process interacting with a formal language constraint.
This develops the theoretical foundations that quantify the distortion introduced by
hard masking relative to true conditional sampling.

\subsection{True conditioning via a Doob $h$-transform}

Fix a base autoregressive model $p(\cdot)$ over token sequences terminated by $\eos$.
Let $L\subseteq \Sigma^*$ be the target language with acceptance event
\[
  E \coloneqq \{\tau(y_{1:T})\in L \text{ and } y_T=\eos\}.
\]

\begin{definition}[Survival function]\label{def:survival}
For any prefix token sequence $y_{<t}$, define the \emph{survival probability}
\[
  h(y_{<t}) \coloneqq \Pr_p(E \mid y_{<t}).
\]
\end{definition}

\begin{theorem}[Doob $h$-transform for conditional decoding]\label{thm:doob}
Assume $h(y_{<t})>0$.
Then the true conditional next-token distribution satisfies
\[
  p_E(v \mid y_{<t}) \coloneqq \Pr_p(y_t=v \mid y_{<t},E)
  = p(v\mid y_{<t})\cdot \frac{h(y_{<t}v)}{h(y_{<t})}.
\]
\end{theorem}
\begin{proof}
By Bayes' rule and the tower property,
\[
  \Pr(y_t=v \mid y_{<t},E)
  =\frac{\Pr(y_t=v \mid y_{<t})\Pr(E \mid y_{<t}v)}{h(y_{<t})}
  =p(v\mid y_{<t})\frac{h(y_{<t}v)}{h(y_{<t})}.
\]
This is the discrete-time analogue of Doob's $h$-transform for Markov chains
\cite{doob1957htransform,norris1997markov}.
\end{proof}

\begin{definition}[Hard-masked process]\label{def:hard-masked}
Let $A(y_{<t})\coloneqq \Omega_G(\tau(y_{<t}))$ be the admissible token set.
Define the hard-masked next-token distribution
\[
  q(v\mid y_{<t}) \coloneqq
  \frac{p(v\mid y_{<t})\,\mathbf{1}[v\in A(y_{<t})]}
       {\sum_{w\in A(y_{<t})} p(w\mid y_{<t})}.
\]
\end{definition}

\subsection{Exact characterization and distortion bounds}

Since $h(y_{<t}v)=0$ whenever $v\notin A(y_{<t})$, both $p_E$ and $q$ have support
contained in $A(y_{<t})$ when $h(y_{<t})>0$.
The distributions differ in how they weight admissible branches.

\begin{theorem}[When hard masking equals true conditioning]
\label{thm:masking-equals-cond}
For a fixed prefix $y_{<t}$ with $h(y_{<t})>0$, the following are equivalent:
\begin{enumerate}
\item $q(\cdot \mid y_{<t}) = p_E(\cdot \mid y_{<t})$,
\item $h(y_{<t}v)$ is constant over all admissible $v\in A(y_{<t})$.
\end{enumerate}
\end{theorem}
\begin{proof}
Restricting to $A=A(y_{<t})$:
\[
  p_E(v\mid y_{<t}) \propto p(v\mid y_{<t})\,h(y_{<t}v),\qquad
  q(v\mid y_{<t}) \propto p(v\mid y_{<t}).
\]
These coincide iff the multiplicative factor $h(y_{<t}v)$ is constant over $A$.
\end{proof}

\begin{definition}[Survival spread]\label{def:spread}
For prefix $y_{<t}$ with admissible set $A=A(y_{<t})$, define
\[
  h_{\min}(y_{<t}) \coloneqq \min_{v\in A} h(y_{<t}v),\qquad
  h_{\max}(y_{<t}) \coloneqq \max_{v\in A} h(y_{<t}v),
\]
and spread ratio $\Gamma(y_{<t})\coloneqq h_{\max}(y_{<t})/h_{\min}(y_{<t})$
(with $\Gamma=\infty$ if $h_{\min}=0$).
\end{definition}

\begin{proposition}[One-step KL bound]\label{prop:kl-bound}
Assume $0<h_{\min}(y_{<t})\le h_{\max}(y_{<t})<\infty$.
Then
\[
  \KL\!\big(q(\cdot\mid y_{<t}) \,\Vert\, p_E(\cdot\mid y_{<t})\big)\;\le\;
  \log \Gamma(y_{<t}).
\]
\end{proposition}
\begin{proof}
Write $A=A(y_{<t})$, $p_v=p(v\mid y_{<t})$, $h_v=h(y_{<t}v)$,
$Z=\sum_{w\in A} p_w$, $H=\sum_{w\in A} p_w h_w$.
Then $q_v=p_v/Z$ and $p^E_v=p_v h_v/H$, giving
\[
  \KL(q\Vert p_E)=\mathbb{E}_q\!\left[\log \frac{q_v}{p^E_v}\right]
  = \log\frac{H}{Z} - \mathbb{E}_q[\log h_v].
\]
Since $h_{\min}\le h_v\le h_{\max}$, we have $\log(H/Z)\le \log h_{\max}$ and
$-\mathbb{E}_q[\log h_v]\le -\log h_{\min}$, so
$\KL(q\Vert p_E)\le \log h_{\max}-\log h_{\min}=\log\Gamma(y_{<t})$.
\end{proof}

\begin{corollary}[One-step total-variation bound]\label{cor:tv-bound}
Under the assumptions of Proposition~\ref{prop:kl-bound},
\[
  \TV\!\big(q(\cdot\mid y_{<t}),\,p_E(\cdot\mid y_{<t})\big)\;\le\;
  \sqrt{\tfrac{1}{2}\log \Gamma(y_{<t})}.
\]
\end{corollary}
\begin{proof}
Apply Pinsker's inequality $\TV(P,Q)\le \sqrt{\KL(P\Vert Q)/2}$ to
Proposition~\ref{prop:kl-bound}.
\end{proof}

\begin{remark}[Interpretation]
Hard masking approximates true conditional sampling when survival probabilities from
admissible next tokens are nearly equal ($\Gamma(y_{<t})\approx 1$).
This formalizes the intuition that hard masking ignores ``ease of completion'' effects
and quantifies distortion purely in terms of completion-probability variability.
The spread $\Gamma$ is large when different admissible tokens lead to vastly different
probabilities of eventual acceptance, and small when the constraint is nearly uniform
over admissible branches.
\end{remark}

% ============================================================
\section{Neural Integration: Attention and Mixture-of-Experts Routing}
\label{sec:neural}
% ============================================================

We express the algebra by which a reachability oracle modulates logits in a Transformer
and extend the formulation to MoE routing probabilities.

\subsection{Transformer logits under grammar-modulated admissibility}

Let a causal Transformer \cite{vaswani2017attention} compute a hidden state
$h_t \in \mathbb{R}^d$ from prefix $y_{<t}$ (with KV-caching assumed).
The unmasked logits are
\[
  \ell_t = W_{\mathrm{out}} h_t + b_{\mathrm{out}} \in \mathbb{R}^{|\Vocab|}.
\]
Let $A_t \coloneqq \Omega_G(\tau(y_{<t}))$ denote the admissible token set.
The masked logits are
\begin{equation}\label{eq:masked-logits}
  \ell_t^{(G)} \coloneqq \ell_t + m_{A_t},
\end{equation}
and the constrained next-token distribution is
\begin{equation}\label{eq:constrained-dist}
  q_\theta^G(y_t=v \mid y_{<t}) = \mathrm{softmax}(\ell_t^{(G)})(v)
  = \frac{\exp(\ell_t(v))\mathbf{1}[v\in A_t]}{\sum_{w\in A_t}\exp(\ell_t(w))}.
\end{equation}
This is the algebraic logit modulation performed by GCD engines
\cite{scholak-etal-2021-picard,pmlr-v235-beurer-kellner24a,park2025flexible,
dong2024xgrammar,ugare2024syncode}.

\subsection{MoE routing under grammar state}

Consider a Transformer with MoE feed-forward blocks
\cite{shazeer2017moe,fedus2022switch}.
At layer $\ell$ and position $t$, the gating network produces pre-activations
$g_{t,\ell}\in\mathbb{R}^E$ over $E$ experts.
Standard routing computes
\[
  r_{t,\ell} = \mathrm{softmax}(g_{t,\ell}),\qquad
  \mathcal{E}_{t,\ell} = \mathrm{TopK}(r_{t,\ell}),\qquad
  h'_{t,\ell} = \sum_{e\in \mathcal{E}_{t,\ell}} \alpha_{t,\ell,e}\,\mathrm{FFN}_e(h_{t,\ell}),
\]
with normalized weights $\alpha_{t,\ell,e}$.
Let $\mathcal{C}_t \coloneqq \CoReach_G(\tau(y_{<t}))$ denote the live configuration
set.
A grammar-aware MoE can incorporate $\mathcal{C}_t$ via a feature embedding
$\phi(\mathcal{C}_t)\in\mathbb{R}^k$:
\begin{equation}\label{eq:moe-routing}
  r_{t,\ell}^{(G)} \;=\; \mathrm{softmax}\!\big(g_{t,\ell} + U_\ell\,\phi(\mathcal{C}_t)\big),
\end{equation}
where $U_\ell\in \mathbb{R}^{E\times k}$ is learned.
This defines a joint constraint coupling syntactic state to conditional computation.

\subsection{Latency envelopes}

Let $V\coloneqq |\Vocab|$ (vocabulary size), $B$ beam width, $t$ the current step,
$\bar{\ell}\coloneqq \mathbb{E}[|\tau(y_t)|]$ the average terminal-length per token,
$M_t$ the live-configuration representation size per hypothesis, and
$K_t \coloneqq |\Omega_G(\tau(y_{<t}))|$ the number of admissible tokens.

\begin{definition}[Abstract engine-step primitives]\label{def:primitives}
Let $\mathrm{Update}_G(M_t,\bar{\ell})$ be the cost to update the grammar engine state
after consuming one token, and $\mathrm{Enumerate}_G(M_t)$ be the cost to enumerate
admissible tokens.
\end{definition}

\begin{proposition}[Per-step complexity: dense-vocabulary masking]
\label{prop:dense-step}
If the model computes dense logits over all $V$ tokens for each of $B$ hypotheses, the
per-step time is
\[
  O\!\big(B\cdot \big( T_{\mathrm{model}}(t) + \mathrm{Update}_G(M_t,\bar{\ell})
  + \mathrm{Enumerate}_G(M_t) + V \big)\big),
\]
where $T_{\mathrm{model}}(t)$ is the Transformer forward cost under KV-caching.
\end{proposition}
\begin{proof}
Per hypothesis: compute logits ($T_{\mathrm{model}}(t)$), update and enumerate mask
($\mathrm{Update}_G+\mathrm{Enumerate}_G$), then scan a $V$-length mask.
Multiply by $B$.
\end{proof}

\begin{proposition}[Per-step complexity: sparse admissible-set scoring]
\label{prop:sparse-step}
If the inference stack scores only admissible tokens (e.g., via trie-restricted
projection), then the $V$ term may be replaced by $K_t$:
\[
  O\!\big(B\cdot \big( T_{\mathrm{model}}(t) + \mathrm{Update}_G(M_t,\bar{\ell})
  + \mathrm{Enumerate}_G(M_t) + K_t \big)\big).
\]
\end{proposition}

\begin{corollary}[SAC-induced masking bottleneck]\label{cor:bottleneck}
For $G_4$ and a sound packed-structure engine, $M_t=\Omega(t^2)$
(Theorem~\ref{thm:lb}), so any update routine that touches newly created packed
structure satisfies $\mathrm{Update}_{G_4}(M_t,\bar{\ell}) = \Omega(t^2)$.
For $G_3$, one can maintain $M_t=O(1)$ giving $\mathrm{Update}_{G_3}=O(1)$.
\end{corollary}

\subsection{Grammar-conditioned logits with sound masking}

\begin{definition}[Grammar-conditioned logits]\label{def:gc-logits}
Let $e_t=\psi(\mathcal{C}_t)\in\mathbb{R}^k$ embed the live configuration set.
Define
\[
  \tilde{\ell}_t(v)\coloneqq \ell_t(v) + u(v)^\top e_t
\]
for learned vectors $u(v)\in\mathbb{R}^k$.
Apply hard masking with $A_t$:
\[
  q_{\theta}^{G,\psi}(v\mid y_{<t}) \coloneqq
  \mathrm{softmax}(\tilde{\ell}_t + m_{A_t})(v).
\]
\end{definition}

\begin{proposition}[Soundness under grammar-conditioned logits]
\label{prop:gc-sound}
Sampling from $q_{\theta}^{G,\psi}$ yields only strings in $\Lang(G)$ upon
termination, for any embedding $\psi$ and weights $u(\cdot)$.
\end{proposition}
\begin{proof}
The support of $q_{\theta}^{G,\psi}$ is contained in $A_t$ at every step by
construction of hard masking.
The argument of Proposition~\ref{prop:soundness} applies unchanged.
\end{proof}

\begin{remark}
Hard masking ensures correctness, while grammar-conditioned logits can reduce the
frequency with which the mask truncates high-probability mass---a key distortion
concern quantified by Proposition~\ref{prop:kl-bound}.
In MoE models, the same embedding $e_t$ can be fed into the router via
\eqref{eq:moe-routing} to encourage expert specialization by syntactic region.
Empirical validation requires controlled ablations on structured benchmarks such as
JSONSchemaBench \cite{geng2025jsonschemabench}.
\end{remark}

% ============================================================
\section{From SAC to Runtime Prediction, Grammar Optimization, and Hybrid Control}
\label{sec:opt}
% ============================================================

This section extends the SAC analysis into (i) an empirically calibratable performance
model, (ii) an equivalence-preserving grammar optimizer, and (iii) deeper
symbolic--neural coupling via instrumented engines.

\subsection{Instrumented engines and SAC proxies}

Production engines do not expose $|F_t(w)|$ directly; they expose engine-native
structures (Earley items, LR items, trie nodes, persistent-stack nodes, etc.).
LLGuidance uses an Earley-based recognizer combined with token-trie traversal
\cite{llguidance_repo}; XGrammar employs vocabulary splitting and persistent-stack
representations \cite{dong2024xgrammar}.
We introduce a general instrumentation interface to unify these.

\begin{definition}[Instrumented decoding engine]\label{def:instrumented}
Fix grammar $G$ and tokenizer $\tau$.
An instrumented decoding engine is a triple $E=(\mathrm{Step},\mathrm{Mask},\mathrm{Ctr})$
where at each step $t$:
\begin{itemize}
\item $\mathrm{Step}$ advances engine state for each hypothesis;
\item $\mathrm{Mask}$ returns $\Omega_G(\tau(y_{<t}))$;
\item $\mathrm{Ctr}$ returns a vector of nonneg.\ counters recording symbolic work
  (items created, edges traversed, trie nodes visited, stack nodes touched, etc.).
\end{itemize}
\end{definition}

\begin{definition}[SAC proxy]\label{def:proxy}
Let $c_t(E;w)\in\mathbb{R}_{\ge 0}^d$ be the counter vector at step $t$.
A scalar \emph{SAC proxy} is
\[
  S_t^{E}(w)\coloneqq \langle \alpha, c_t(E;w)\rangle
\]
for some fixed $\alpha\in\mathbb{R}_{\ge 0}^d$, chosen to estimate marginal per-step
symbolic work.
\end{definition}

\begin{lemma}[Proxy lower-bounds SAC under split completeness]\label{lem:proxy-lb}
Consider $G_4$ with rule $S\to SS$.
Assume engine $E$ is \emph{split-complete}: for every span $[i,t)$ of length $\ge 2$
in its maintained structure, it represents each split point $k$ with $i<k<t$.
Then there exists $\beta>0$ such that for all $t$,
\[
  \max_{w\in \Sigma^t} S_t^E(w) \;\ge\; \beta \cdot \mathrm{SAC}_{G_4}(t).
\]
\end{lemma}
\begin{proof}
At step $t$, split completeness forces the engine to account for each newly introduced
split alternative for each new span ending at $t$.
By Theorem~\ref{thm:sac-sep}, this is $\Omega(t^2)$ in the worst case.
Any proxy assigning positive weight to such alternatives yields a linear lower bound
on that count up to a constant.
\end{proof}

\subsection{A predictive performance model}

Let $T_{\mathrm{NN}}(t)$ denote GPU model forward time per hypothesis under KV-cache,
and $T_{\mathrm{mask}}(t)$ the CPU-side (or mixed) engine time for state advancement
and mask computation.

\begin{definition}[Critical-path step time]\label{def:critical-path}
Under ideal overlap,
\[
  T_{\mathrm{step}}(t) \coloneqq
  \max\{T_{\mathrm{NN}}(t), \, T_{\mathrm{mask}}(t)\} + T_{\mathrm{sync}}(t)
  + T_{\mathrm{sel}}(t),
\]
where $T_{\mathrm{sync}}(t)$ accounts for synchronization and $T_{\mathrm{sel}}(t)$
accounts for sampling/top-$k$ selection.
\end{definition}

XGrammar explicitly targets ``overlap grammar computation with GPU execution'' as a
latency optimization strategy \cite{dong2024xgrammar}.
A practical study should measure overlap and identify when masking becomes the critical
path; Nsight Systems provides NVTX range annotations and timeline visualization for
this purpose \cite{nsight_systems_userguide}.

\begin{definition}[Affine proxy-to-time model]\label{def:affine-model}
An affine proxy model is
\[
  T_{\mathrm{mask}}(t) \approx a\cdot S_t^E(w) + b,
\]
for constants $a\ge 0$, $b\ge 0$ estimated from empirical traces via, e.g.,
least-squares fit.
\end{definition}

\begin{remark}[Calibration toolchain]
Calibration is naturally performed on mask-only workloads (e.g., MaskBench)
\cite{jsonschemabench_repo}.
A complete toolchain combines: Nsight Systems for timeline/overlap analysis
\cite{nsight_systems_userguide}; Linux \texttt{perf} for CPU PMU counters
\cite{perfwiki_tutorial}; Nsight Compute for GPU cache/memory analysis
\cite{nsight_compute_profilingguide}; PyTorch Profiler for operator-level attribution
\cite{pytorch_profiler_recipe}; and CUDA allocator snapshots for memory pressure
\cite{pytorch_cuda_memory}.
\end{remark}

\begin{proposition}[Beam amplification of symbolic work]\label{prop:beam}
Assuming beam hypotheses do not share grammar engine state,
\[
  S^{E}_{t,\mathrm{total}} \;=\; \sum_{i=1}^{B} S^{E}_t(w^{(i)}) \;=\; \Theta(B\cdot \overline{S}_t),
\]
where $\overline{S}_t$ is the average proxy per hypothesis.
Under the affine model, $T_{\mathrm{mask,total}}(t) \approx a\cdot \Theta(B\cdot\overline{S}_t)+b'$.
\end{proposition}
\begin{proof}
Counters are generated by disjoint engine updates; the proxy is additive across
independent hypotheses.
Substituting into the affine model absorbs constant terms into $b'$.
\end{proof}

\begin{remark}[DPDA compilation and batch efficiency]
When nondeterminism forces persistent-stack branching, $S_t^E$ can grow with beam even
at fixed prefix length.
Pre$^3$ targets this by compiling LR(1) structures to a DPDA and precomputing
prefix-conditioned edges to reduce runtime path exploration \cite{chen-etal-2025-pre3}.
Within our framework this corresponds to reducing active configuration fanout and
shrinking variance in $S_t^E$ across steps.
\end{remark}

\subsection{Automated grammar optimization}

Oracle invariance (Theorem~\ref{thm:oracle}) permits a compiler searching over
equivalent grammars to minimize an engine-dependent cost model.

\begin{definition}[Grammar optimization problem]\label{def:opt}
Fix engine family $E$, tokenizer $\tau$, and workload $\mathcal{D}$.
Define the feasible set $\mathcal{G}(G)\coloneqq \{G' : \Lang(G')=\Lang(G)\}$ and
objective
\[
  \min_{G'\in \mathcal{G}(G)} \;\; \mathbb{E}_{w\sim \mathcal{D}}
  \Big[\mathrm{Cost}_E(G',\tau;w)\Big],
\]
where $\mathrm{Cost}_E$ can combine static terms (e.g., $\kappa(G')$), dynamic terms
(expected SAC proxy), and tokenizer alignment overhead.
\end{definition}

\begin{remark}[Tokenizer alignment]
DOMINO highlights that improper alignment between subword vocabularies and grammar
terminals can distort model distributions and incur ``bridge'' handling overhead
\cite{pmlr-v235-beurer-kellner24a}.
Park et al.\ formalize lexer-state-dependent token-to-terminal mapping to avoid
inefficient combinations at runtime \cite{park2025flexible}.
Thus misalignment cost is a first-class component of $\mathrm{Cost}_E$.
\end{remark}

\begin{definition}[Inlining rewrite]\label{def:inlining}
Let $A\in N$ have productions $A\to \gamma_1 \mid \cdots \mid \gamma_m$.
An \emph{inlining} rewrite replaces a production $B\to \alpha A \beta$ by the set
$\{B\to \alpha \gamma_i \beta : i=1,\dots,m\}$, leaving other productions unchanged.
\end{definition}

\begin{lemma}[Inlining preserves the language]\label{lem:inline}
Assume $A$ is not the start symbol and the inlining rewrite introduces no new
nonterminals.
Then the rewritten grammar $G'$ satisfies $\Lang(G')=\Lang(G)$.
\end{lemma}
\begin{proof}
Each derivation in $G$ applying $B\to \alpha A \beta$ followed by $A\Rightarrow^*\gamma_i$
maps to a derivation in $G'$ using $B\to \alpha\gamma_i\beta$; thus
$\Lang(G)\subseteq \Lang(G')$.
The reverse containment is symmetric.
\end{proof}

\begin{remark}[Equality saturation for grammar optimization]
Even with local rewrites, the search space of equivalent grammars is large.
Equality saturation with e-graphs \cite{tate2009eqsat,willsey2021egg} offers a
principled approach: represent grammar fragments in an e-graph, apply
equivalence-preserving rewrites (inlining, delegation elimination, recursion
normalization), attach cost analyses estimating $\kappa$, SAC proxy surrogates, and
tokenizer alignment, then extract the minimum-cost equivalent grammar.
Hardness results (NP-completeness of VPA minimization \cite{gauwin2020vpa}) motivate
heuristic or class-restricted optimization rather than global optimality.
\end{remark}

% ============================================================
\section{Related Work}\label{sec:related}
% ============================================================

\paragraph{Parsing theory and reachability.}
Classic CFG parsing algorithms provide the backbone for our formalization
\cite{earley1970efficient,younger1967recognition}.
GLR parsing and ambiguity complexity are developed in \cite{johnson1991glr}, and
SPPF-based packed forest representations in \cite{scott2010sppf}.
Valiant's reduction of CFG recognition to Boolean matrix multiplication
\cite{valiant1975cfg} and Lee's converse reduction \cite{lee2002bmm} establish tight
connections between general CFG parsing complexity and matrix multiplication; our
engine-independent lower bounds (Section~\ref{sec:lb}) are complementary and
output-structure driven.
CFL reachability and pushdown model checking are developed in
\cite{reps1995graphreachability,bouajjani1997reachability,alur2005rsm}.

\paragraph{Grammar-constrained decoding.}
PICARD uses incremental parsing to reject invalid subword tokens during decoding
\cite{scholak-etal-2021-picard}.
DOMINO emphasizes subword-aligned constraining and non-invasive masking
\cite{pmlr-v235-beurer-kellner24a}.
Park et al.\ analyze preprocessing/online tradeoffs and lexer-state-dependent token
mapping \cite{park2025flexible}.
XGrammar introduces vocabulary splitting, persistent stacks, grammar transformations,
and CPU/GPU overlap \cite{dong2024xgrammar}.
SynCode provides a completeness/soundness framing \cite{ugare2024syncode}.
Pre$^3$ compiles deterministic pushdown automata to reduce runtime path exploration in
large-batch settings \cite{chen-etal-2025-pre3}.
JSONSchemaBench provides a large-scale structured-output benchmark with a MaskBench
component isolating mask computation performance
\cite{geng2025jsonschemabench,jsonschemabench_repo}.
Libraries such as Guidance, Outlines, and LLGuidance provide widely used constraint
interfaces \cite{guidance,guidance_github,outlines_repo,llguidance_repo}.

\paragraph{Stochastic processes and $h$-transforms.}
The theory of Doob $h$-transforms characterizes conditional Markov processes and
bridges probabilistic analysis with harmonic function theory
\cite{doob1957htransform,norris1997markov}.
Our use of the $h$-transform to characterize GCD as true conditional decoding
(Section~\ref{sec:stoch}) connects constrained generation to the classical theory of
conditioning on rare events.

\paragraph{Grammar optimization.}
Equality saturation and e-graphs \cite{tate2009eqsat,willsey2021egg} provide a
principled framework for optimization under equivalence, which we propose as a basis
for automated grammar refactoring.
NP-completeness of visibly pushdown automaton minimization \cite{gauwin2020vpa}
motivates heuristic and class-restricted approaches.

\paragraph{Neural architecture coupling.}
The Transformer \cite{vaswani2017attention} and MoE architectures
\cite{shazeer2017moe,fedus2022switch} are the inference targets for our neural
integration results.
Constrained generation in semantic parsing is studied in \cite{shin-etal-2021-constrained},
and LMQL provides a programming interface over constrained LLM decoding
\cite{beurerkellner2023lmql}.

% ============================================================
\section{Conclusion}\label{sec:conclusion}
% ============================================================

We formalized grammar-constrained decoding as a coupling between a neural next-token
model and a pushdown reachability oracle derived from a CFG.
Our contributions are as follows.

\emph{Oracle invariance} (Theorem~\ref{thm:oracle}): language-equivalent grammars
induce identical admissible token sets, hence identical logit masks, for every prefix.
This invariance coexists with provable differences in engine state-space size and online
update cost.

\emph{State-space blowup} (Lemma~\ref{lem:blowup}): for the $a^n b^n$ language, we
gave exact control-state counts under a concrete PDA compilation and showed a $15/8$
blowup factor from redundant nonterminal delegation.

\emph{SAC bounds} (Theorem~\ref{thm:sac-sep}, Corollary~\ref{cor:cumulative}):
concatenation grammars incur $\Theta(t^2)$ incremental packed-structure growth and
$\Theta(n^3)$ cumulative cost; right-recursive grammars admit $O(1)$ and $O(n)$,
respectively.

\emph{Engine-independent lower bounds} (Theorem~\ref{thm:lb}): any sound,
retrieval-efficient, parse-preserving masking engine must incur $\Omega(t^2)$ per-token
work on $G_4$, making the SAC lower bound unconditional within this model.

\emph{Decoding-cost equivalence classes} (Theorem~\ref{thm:min-rep}): minimal-SAC
grammar representatives exist within any bounded-rewrite family, providing a
mathematically precise notion of canonical low-SAC normal forms.

\emph{Grammar-conditioned processes} (Theorem~\ref{thm:doob},
Proposition~\ref{prop:kl-bound}, Corollary~\ref{cor:tv-bound}): the true conditional
sampler is characterized by a Doob $h$-transform, and hard masking incurs a one-step
KL distortion bounded by $\log \Gamma(y_{<t})$, where $\Gamma$ measures
survival-probability spread among admissible next tokens.

\emph{Neural integration and performance modeling}: we derived latency envelopes
for Transformer and MoE architectures, showed that SAC-induced bottlenecks affect
masking-critical-path time under beam search, and connected these results to
instrumentation-based predictive models and automated grammar optimization via
equivalence-preserving rewrites and equality saturation
\cite{geng2025jsonschemabench,chen-etal-2025-pre3,tate2009eqsat,willsey2021egg}.

Together, these results establish a rigorous theoretical foundation for grammar
refactoring as a latency optimization problem with well-defined semantic invariants and
measurable cost surrogates.

\end{document}